\documentclass[letterpaper, 10 pt, conference]{ieeeconf}  

\IEEEoverridecommandlockouts                              

\usepackage{times} 
\usepackage{amsmath}
\usepackage{bm}
\usepackage{amssymb}
\usepackage{amsfonts}
\usepackage[dvipsnames]{xcolor}
\usepackage[utf8]{inputenc}
\usepackage{graphicx}
\usepackage[hyphens]{url}
\usepackage{soul}
\usepackage{booktabs}
\usepackage[ruled,vlined, linesnumbered]{algorithm2e}
\usepackage{multirow}
\usepackage[T1]{fontenc}
\usepackage{multicol}
\usepackage{threeparttable} 
\usepackage{bbm}
\usepackage{hyperref}
\usepackage{graphicx}

\newcommand{\revise}[1]{\textcolor{black}{#1}}

\newcommand{\sectref}[1]{Section~\ref{#1}}

\newcommand{\figref}[1]{Figure~\ref{#1}}
\newcommand{\tabref}[1]{Table~\ref{#1}}
\newcommand{\egref}[1]{Example~\ref{#1}}
\newcommand{\eqnref}[1]{Equation~\ref{#1}}
\newcommand{\thmref}[1]{Theorem~\ref{#1}}

\newcommand{\lemref}[1]{Lemma~\ref{#1}}

\newcommand{\agref}[1]{Algorithm~\ref{#1}}

\newcommand{\eqneqnref}[2]{Equation~\ref{#1} and \ref{#2}}

\newtheorem{lemma}{Lemma}

\newtheorem{theorem}{Theorem}

\newcommand{\uitstart}[1]{{\noindent \it \underline{#1}}}
\newcommand{\startpara}[1]{{\vskip5pt\noindent{\bf #1.}}} 
\renewcommand{\url}[1]{{\def~{\char126}\sf#1}}

\newcounter{exampcount}
\newenvironment{examp}
{\refstepcounter{exampcount}
\vskip6pt\noindent
{\bf Example \arabic{exampcount}.}}
{\hfill$\blacksquare$\vskip6pt}

\DeclareMathOperator*{\argmax}{arg\,max}

\def\Rset{\mathbb{R}}

\def\cT{{\mathcal{T}}}

\def\cG{{\mathcal{G}}}

\def\Pr{{\mathit{Pr}}}

\newcommand\post[2]{\mathsf{post}({#1},{#2})}
\def\obs{{\mathsf{obs}}}
\def\supp{{\mathit{supp}}}
\def\post{{\mathrm{post}}}
\def\indicator{\mathbb{I}}
\def\nset{{\mathcal{N}}}
\def\vset{\mathcal{V}}
\def\pset{\mathcal{B}}

\newcommand{\D}{\mathcal{D}}

\newcommand{\M}{\mathcal{M}} 

\renewcommand{\O}{\mathcal{O}}

\title{\LARGE \bf Safe POMDP Online Planning among Dynamic Agents \\ via Adaptive Conformal Prediction}

\author{Shili Sheng$^{1}$, Pian Yu$^{2}$,  David Parker$^{2}$, Marta Kwiatkowska$^{2}$, and Lu Feng$^{1}$
\thanks{*This work was supported in part by NSF grants CCF-1942836 and CCF-2131511, and the ERC ADG FUN2MODEL (Grant agreement ID: 834115). }
\thanks{$^{1}$Shili Sheng and Lu Feng are with the School of Engineering and Applied Science, University of Virginia, Charlottesville, VA 22904, USA {\tt\small  \{ss7dr, lu.feng\}@virginia.edu}}%
\thanks{$^{2}$Pian Yu, David Parker, and Marta Kwiatkowska are with the Department of Computer Science, University of Oxford, Parks Road, Oxford OX1 3QD, United Kingdom {\tt\small \{pian.yu, david.parker, marta.kwiatkowska\}@cs.ox.ac.uk}}%
}

\begin{document}
\maketitle

\begin{abstract}
Online planning for partially observable Markov decision processes (POMDPs) provides efficient techniques for robot decision-making under uncertainty. However, existing methods fall short of preventing safety violations in dynamic environments. This work presents a novel safe POMDP online planning approach that maximizes expected returns while providing probabilistic safety guarantees amidst environments populated by multiple dynamic agents. Our approach utilizes data-driven trajectory prediction models of dynamic agents and applies Adaptive Conformal Prediction (ACP) to quantify the uncertainties in these predictions. Leveraging the obtained ACP-based trajectory predictions, our approach constructs safety shields on-the-fly to prevent unsafe actions within POMDP online planning. Through experimental evaluation in various dynamic environments using real-world pedestrian trajectory data, the proposed approach has been shown to effectively maintain probabilistic safety guarantees while accommodating up to hundreds of dynamic agents.
\end{abstract}

\section{Introduction} \label{sec:intro} 

The \emph{partially observable Markov decision process} (POMDP) framework is a general model for decision making under uncertainty~\cite{lauri2022partially}, which finds application in various robotic tasks, such as autonomous driving~\cite{sheng2022planning} and human-robot collaboration~\cite{yu2024trust}. 
Significant progress in POMDP online planning, which interleaves policy computation and execution, has been made to overcome computational challenges. 
For instance, the widely adopted Partially Observable Monte Carlo Planning (POMCP) algorithm~\cite{silver2010monte} enhances scalability through Monte Carlo sampling and simulation.

For many safety-critical robotic applications, computing POMDP policies that satisfy safety requirements is crucial. 
Existing methods for safe POMDP online planning often represent safety requirements as cost or chance constraints, aiming to maximize expected returns while reducing cumulative costs or failure probabilities~\cite{lee2018monte,khonji2019approximability}. 
However, these methods cannot guarantee complete prevention of safety violations.
In our previous work~\cite{sheng2024safe}, we integrated POMCP with safety shields to ensure that, with probability one, goal states are reached and unsafe states are avoided. 
But these shielding methods are limited to static obstacles and fall short in dynamic environments.

To tackle this limitation, in this work we investigate safe POMDP online planning for a robotic agent travelling among multiple unknown dynamic agents, such as pedestrians or other robots. 
We consider a safety constraint, which specifies that the minimum distance from the robotic agent to any of the dynamic agents should exceed a predefined safety buffer. 
The goal is to develop a safe POMDP online planning method that computes an optimal policy maximizing the expected return while ensuring that the probability of satisfying the safety constraint exceeds a certain threshold.

This work addresses several key challenges. 
The first is the modeling of dynamic agents. We use data-driven trajectory models to predict the movements of these agents and apply Adaptive Conformal Prediction (ACP) to quantify the uncertainties in these predictions, as per~\cite{dixit2023adaptive}.
The second challenge involves the construction of safety shields that avert collisions with dynamic agents. 
To this end, we propose a novel algorithm that dynamically constructs safety shields to accommodate the ACP-based prediction regions of dynamic agents.
The third challenge is integrating safety shields into POMDP online planning. 
We enhance the POMCP algorithm with safety shields by evaluating the safety of each action during the Monte Carlo sampling and simulation process.

To the best of our knowledge, this is the first safe POMDP online planning approach that offers probabilistic safety guarantees in environments with dynamic agents.
We evaluate the proposed approach through computational experiments in various dynamic environments, utilizing real-world pedestrian trajectory data.

\subsection{Related Work} \label{sec:related} 
\startpara{Safe POMDP online planning}
Prior studies have explored different methods for incorporating safety constraints into online planning with POMDPs.
Online algorithms for constrained POMDPs apply cost (or chance) constraints to limit expected cumulative costs (or failure probability); however, they do not guarantee the avoidance of constraint violations~\cite{lee2018monte,khonji2019approximability}. 
An online method introduced in~\cite{wang2021online} synthesizes a partial conditional plan for POMDPs with a focus on safe reachability, aiming for specific probability thresholds to reach goals or avoid static obstacles.
Additionally, a rule-based shielding method presented in~\cite{mazzi2023risk} generates shields by learning parameters for expert-defined rule sets. 
In our previous work~\cite{sheng2024safe}, we devised shields to preemptively prevent unsafe actions and incorporated these shields into the POMCP algorithm to ensure safe online planning for POMDPs. 
However, these methods primarily focus on circumventing static obstacles and are not directly applicable to safe planning in dynamic environments. \revise{In \cite{kurniawati2016online}, an online solver called Adaptive Belief Tree (ABT) was proposed for POMDP planning in dynamic environments, which updates the POMDP model in response to changes in the environment. However, this work does not address how to identify these changes in the POMDP model.}

\startpara{Planning among dynamic agents}
Substantial research has been conducted in the field of robotic planning involving dynamic agents, such as pedestrians. 
Some works treat these dynamic agents as static obstacles, adapting to changes through online replanning~\cite{bauer2009autonomous}. 
Other approaches make simple assumptions about the dynamics of these agents; for example, they may assume that pedestrians move at a constant velocity~\cite{kummerle2015autonomous}. 
There are also more sophisticated methods that model the intentions of dynamic agents. 
For instance, PORCA~\cite{luo2018porca} is a POMDP-based planning method that accounts for complex models of pedestrians' intentions and interactions.
Additionally, data-driven trajectory predictors have seen extensive use in existing work (e.g., \cite{zhu2023gaussian,farid2023task}).
However, numerous data-driven prediction techniques, such as Long Short-Term Memory (LSTM), often lack mechanisms to convey uncertainty in their predictions, potentially resulting in decisions that compromise safety.
In this work, we adopt LSTM-based trajectory predictors and enhance them by incorporating uncertainty estimation through conformal prediction.

\startpara{Planning with conformal prediction}
Conformal prediction offers techniques for estimating statistically rigorous uncertainty sets for predictive models, such as neural networks, without making assumptions about the underlying distributions or models~\cite{angelopoulos2023conformal}. Adaptive Conformal Prediction (ACP) extends these techniques to estimate prediction regions for time series data~\cite{zaffran2022adaptive}. Recently, there has been a surge in integrating conformal prediction, including its adaptive variant, into planning frameworks to accommodate the uncertainty in predicted trajectories. For example, methods based on Model Predictive Control (MPC) have been developed for safe planning in dynamic environments, which incorporate (adaptive) conformal prediction regions of the predicted trajectories of dynamic agents~\cite{lindemann2023safe,dixit2023adaptive}.
Furthermore, conformal prediction has been applied to quantify the uncertainty in trajectory predictions derived from diffusion dynamics models, aiding in planning and offline reinforcement learning applications~\cite{sun2024conformal}. \revise{Recent work \cite{moss2024constrainedzero} considered chance-constrained POMDPs, and used adaptive conformal inference for estimation of failure probability thresholds.}
In the realm of large language models, conformal prediction has been leveraged to provide statistical guarantees on the completion of robot tasks, enhancing the reliability of language model-based planners~\cite{ren2023robots}.

In this work, we adopt the ACP-based trajectory predictor proposed in~\cite{dixit2023adaptive} and develop a safe POMDP online planning method that constructs shields on-the-fly incorporating ACP prediction regions of dynamic agents.

\section{Problem Formulation} \label{sec:problem} 

\startpara{POMDP model}
We model the dynamics of a robotic agent as a POMDP, denoted as a tuple $\M=(S, A, O, T, R, Z, \gamma)$,
where $S$, $A$ and $O$ are (finite) sets of states, actions, and observations, respectively;
$T: S \times A \times S \to [0,1]$ is the probabilistic transition function;
$R: S \times A \to \Rset$ is the reward function;
$Z: S \times A \times O \to [0,1]$ is the observation function; 
and $\gamma \in [0,1]$ is the discount factor.
At each timestep $t$, the state $s_t \in S$ transitions to a successor state $s_{t+1} \in S$ with probability $T(s_t, a_t, s_{t+1}) = \Pr(s_{t+1} \mid s_t,a_t)$ given an agent's action $a_t \in A$; the agent receives a reward $R(s_t, a_t)$, and makes an observation $o_{t+1} \in O$ about state $s_{t+1}$ with probability $Z(s_{t+1}, a_t, o_{t+1})=\Pr(o_{t+1} \mid s_{t+1}, a_t)$. 

Given the partial observability of POMDP states, the agent maintains a \emph{history} of actions and observations, denoted by $h_t=a_0, o_1, \dots, a_{t-1}, o_t$.
A \emph{belief state} represents the posterior probability distribution over states conditioned on the history, denoted by $b_t(s) = \Pr(s_t=s | h_t)$ for $s \in S$. 
Let $b_0$ denote the initial belief state, representing a distribution over the POMDP's initial states. 
Let $B$ be the set of belief states of POMDP $\M$. 
The \emph{belief support} of a belief state $b \in B$ is defined as $\supp(b) := \{ s  \in S | b(s) > 0 \}$, i.e., the set of states with positive belief.
The set of belief supports of POMDP $\M$ is defined as $S_B:=\{\Theta \subseteq S \mid \forall s, s' \in \Theta, \obs(s) = \obs(s')\}$, 
where $\obs: S \to 2^O$ is a function representing the set of possible observations for a state.

A POMDP \emph{policy}, denoted by $\pi: B \to A$, is a mapping from belief states to actions. 
At timestep $t$, executing a policy $\pi$ involves selecting an action $a_t = \pi(b_t)$ based on the current belief state $b_t$, and subsequently updating to belief state $b_{t+1}$ after observing $o_{t+1}$ according to Bayes' rule:
\begin{equation}
b_{t+1}(s')=\frac{Z(s', a_t, o_{t+1}) \sum_{s \in S} T(s, a_t, s')b_t(s)}{\eta(o_{t+1} \mid b,a)}
\end{equation}
where $\eta(o_{t+1} \mid b,a)$ is a normalizing constant representing the prior probability of observing $o_{t+1}$.

Let $R(b_t, a_t):= \sum_{s \in S} R(s,a_t)b_t(s)$ denote the expected immediate reward of taking action $a_t$ in belief state $b_t$. 
The expected return from following policy $\pi$ starting at initial belief state $b_0$ is defined as:
\begin{equation}
    V^{\pi}(b_0) := \mathbb{E}_\pi [\sum_{t=0}^\infty \gamma^t R\left(b_t, \pi(b_t)\right) \ | \ b_0].
\end{equation}

\startpara{Dynamic agents}
Consider a robot operating in an environment with $N$ dynamic agents whose trajectories are \emph{a priori} unknown. 
Let $X_t := (X_{t,1}, \dots, X_{t,N})$ denote the joint agent state at timestep $t$, where $X_{t,i}$ represents the $i$-th dynamic agent's state. 
\revise{In this work, we assume that $X_t$ is fully observable to the robot at timestep $t$.}
Assume the agents' trajectories adhere to an unknown distribution $\D$.
Let $X := (X_0, X_1, \dots) \sim \D$ be a random trajectory sampled from this distribution.
We represent the safety constraint, which requires the minimum distance from the robot to any of the $N$ dynamic agents to exceed a safety buffer $\epsilon \in \mathbb{R}^+$, through the following Lipschitz continuous constraint function:
\begin{equation} \label{eqn:safety}
    c(s_t, X_t) := \min_{i \in \{1,\dots, N\}} \|s_t - X_{t,i}\| - \epsilon.
\end{equation}

Given a belief state $b_t$ of the POMDP $\M$, we compute the probability of $b_t$ satisfying the safety constraint by summing over the probabilities $b_t(s)$ for all states $s \in \supp(b_t)$ that meet the condition $c(s, X_t) \ge 0$, denoted as:
\begin{equation} \label{eqn:prob-safe}
    \rho(b_t, X_t) := \sum_{s \in \supp(b_t)} b_t(s) \cdot \indicator_{\{c(s, X_t) \ge 0\}}.
\end{equation}

We define the expected average probability of satisfying the safety constraint across all potential POMDP executions initiated from the initial belief state $b_0$ under a policy $\pi$ as:
\begin{equation}
    \phi^{\pi}(b_0) := \mathbb{E}_{\pi} [\lim_{T\to \infty}\frac{1}{T}\sum_{t=0}^{T-1} \rho(b_t, X_t) \ | \ b_0, X_t \sim \D]. 
\end{equation}

\startpara{Problem}
Given a POMDP model $\M$ for a robotic agent with initial belief state $b_0$, unknown random trajectories $X \sim \D$ of $N$ dynamic agents, and a failure rate $\delta \in (0,1)$, the objective is to compute an optimal POMDP policy $\pi^*$ that maximizes the expected return $V^{\pi}(b_0)$ while ensuring that the expected average probability of satisfying the safety constraint is at least $1-\delta$, denoted by $\phi^{\pi}(b_0) \ge 1 - \delta$.

\section{Preliminaries} \label{sec:prelim} 
To tackle the problem under consideration, we propose a safe online planning approach for POMDPs that accounts for the uncertainty of dynamic agents' predicted trajectories. The key idea is as follows. Initially, we convert the probabilistic safety constraint concerning dynamic agents, i.e., $\phi^{\pi}(b_0) \ge 1 - \delta$, into an equivalent almost-sure safety constraint. This conversion is facilitated through Adaptive Conformal Prediction (ACP), a technique capable of adaptively quantifying the uncertainty of trajectory predictors and generating prediction regions with predefined probability thresholds. Subsequently, safety shields are constructed
to accommodate the previously computed prediction regions of the dynamic agents, and they are integrated into the Partially Observable Monte Carlo Planning (POMCP) algorithm for online planning. 

We now present the essential preliminaries on ACP for trajectory prediction in \sectref{sec:acp}, and the POMCP algorithm for online planning in \sectref{sec:pomcp}.

\subsection{ACP-based Trajectory Prediction} \label{sec:acp}

Assume there exists a trajectory predictor capable of making predictions about the future trajectories of dynamic agents for a finite horizon $H$ based on their past trajectories. Though our proposed approach is agnostic to the prediction method, we employ an LSTM model as the trajectory predictor in this work and make no additional assumptions about the trajectory distribution $\mathcal{D}$. 
To consider the uncertainty of predicted trajectories, which could influence the satisfaction of safety constraints, we compute ACP prediction regions using the method described in~\cite{dixit2023adaptive}.

Let $(\hat{X}_t^1, \dots, \hat{X}_t^H)$ represent the predicted trajectory of dynamic agents' future states starting at timestep $t$ and extending to the prediction horizon $H$, 
where $\hat{X}_t^\tau:= (\hat{X}_{t,1}^\tau, \dots, \hat{X}_{t,N}^\tau)$ is the predicted joint state of $N$ agents made at timestep $t$ for horizon $\tau \in \{1, \dots, H\}$. 
However, we cannot evaluate the prediction error for future states since the ground truths $(X_{t + 1},\dots, X_{t + H})$ are unknown at timestep $t$.
We adopt the concept of \emph{time-lagged nonconformity score}, as defined in~\cite{dixit2023adaptive}, which quantifies the $\tau$ step-ahead prediction error made $\tau$ timesteps ago, denoted by $\beta_t^{\tau}:= \|X_t - \hat{X}_{t - \tau}^{\tau} \|$.

For each prediction horizon $\tau \in \{1, \dots, H\}$, we calculate an ACP prediction region $\beta_{t}^{\tau} \le C_{t}^{\tau}$ based on $(\beta_{t-K}^{\tau}, \dots, \beta_{t-1}^{\tau})$ with a sliding window of size $K$, ensuring $\Pr(\beta_{t}^{\tau} \le C_{t}^{\tau}) \ge 1-\delta$, where $\delta \in (0,1)$ denotes a failure probability. The value of $C_{t}^{\tau}$ is determined by identifying the $\lceil(K+1)(1-\lambda^\tau_{t})\rceil^{\text{th}}$ smallest value among $(\beta_{t-K}^{\tau},\dots,\beta_{t-1}^{\tau})$, with the parameter $\lambda^\tau_{t}$ updated recursively as follows:
\begin{equation}\label{eq:failureupdate}
\lambda_{t}^{\tau}:= \lambda_{t-1}^{\tau} + \alpha(\delta - \indicator_{\{C_{t-1}^{\tau} < \beta_{t-1}^{\tau}\}})
\end{equation}
where $\alpha \in (0,1)$ is the learning rate and $\lambda_0^{1} \in (0,1)$ is a constant for the initial value.

\begin{figure}[t]
    \centering
    \includegraphics[width=.9\columnwidth]{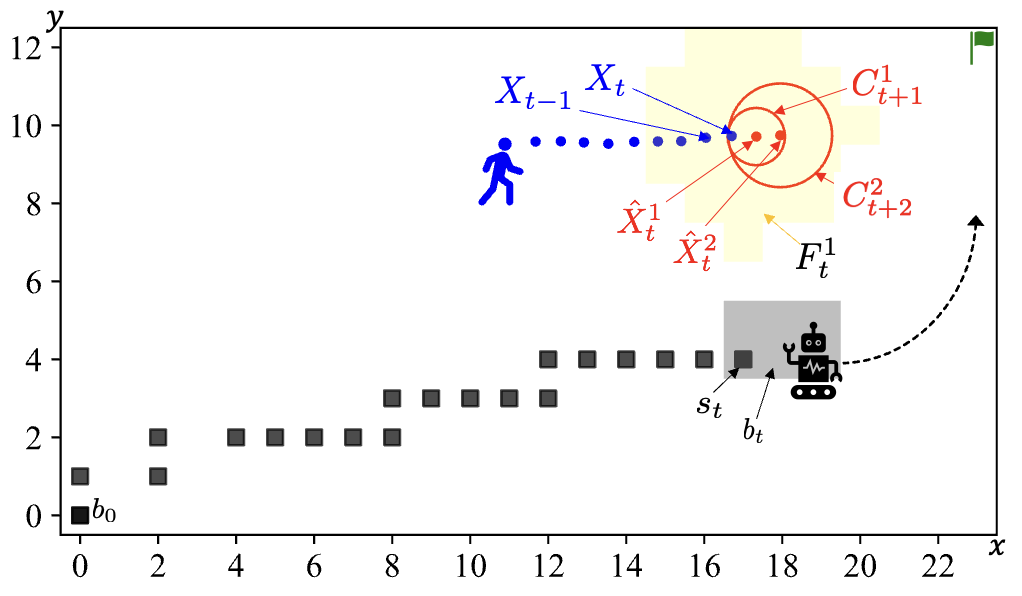}
    \caption{Example gridworld with a robot navigating towards a flag while avoiding a pedestrian. \revise{The robot} moves \emph{east}, \emph{south}, \emph{west}, or \emph{north}, reaching the adjacent grid cell with probability 0.1 or one cell further with probability 0.9. Gray shadow: robot's belief state $b_t$ including state $s_t$. Red circles: ACP prediction regions of uncertain predictions about pedestrian states. Yellow shadow: unsafe states per one-step prediction at timestep $t$.}
    \label{fig:example}
\end{figure}

\begin{examp}\label{eg:acp}
Consider a robot navigating in a gridworld with one dynamic agent (a pedestrian) as shown in \figref{fig:example}. 
Let $\langle x,y\rangle$ denote a two-dimensional position in the gridworld. 
At timestep $t$, the dynamic agent's (joint) state is $X_t=(\langle 16.702, 9.726 \rangle)$.
A trajectory predictor with the prediction horizon $H=2$ yields one-step and two-step ahead predictions as 
$\hat{X}_t^1=(\langle 17.334, 9.711 \rangle)$ and $\hat{X}_t^2=(\langle17.947, 9.743\rangle)$, respectively.
We compute the time-lagged nonconformity score $\beta_t^{1} = \|X_t-\hat{X}_{t-1}^1\|=0.068$,
where $\hat{X}_{t-1}^1 = (\langle 16.650, 9.682 \rangle)$ is the prediction about state $X_t$ made one timestep ago at $t-1$. 

Next we compute the ACP prediction region $C_{t+1}^{1}$ to ensure $\Pr(\beta_{t+1}^{1} \le C_{t+1}^{1}) \ge 1-\delta$ such that 
the prediction error $\beta_{t+1}^{1} = \|X_{t+1}-\hat{X}_t^1\|$ is bounded with the failure probability $\delta = 0.05$. 
Let the size of the sliding window be $K=30$ and the learning rate be $\alpha=0.0008$.
Suppose $C_t^{1} = 0.736$ and $\lambda_t^{1}=0.0495$. 
We have $\lambda_{t+1}^{1} = \lambda_t^{1} + \alpha(\delta - \indicator_{\{C_t^{1} < \beta_t^{1}\}}) = 0.04954$.
We determine the value of $C_{t+1}^{1}$ by finding the $\lceil(K+1)(1-\lambda^1_{t+1})\rceil^{\text{th}}$, which is the $30^{\text{th}}$ smallest value among $(\beta_{t-29}^{1},\dots,\beta_t^{1})$. This process yields $C_{t+1}^{1}=0.736$. 
Similarly, the ACP prediction region to bound the prediction error $\beta_{t+2}^{2} = \|X_{t+2}-\hat{X}_t^2\|$ is computed as $C_{t+2}^{2}=1.329$. 
The obtained ACP prediction regions are plotted as red circles in \figref{fig:example}, centered at $\hat{X}_t^1$ and $\hat{X}_t^2$, with radii $C_{t+1}^{1}$ and $C_{t+2}^{2}$, respectively.
\end{examp}

\subsection{Partially Observable Monte-Carlo Planning} \label{sec:pomcp}

We employ the \emph{Partially Observable Monte Carlo Planning} (POMCP) algorithm~\cite{silver2010monte}, a popular method for POMDP online planning that interleaves the policy computation and execution. 
Each timestep $t$ starts with the POMCP algorithm deploying Monte Carlo tree search~\cite{coulom2006efficient} to navigate a search tree. The root node of this tree is $\cT(h_t)=\langle \nset(h_t), \vset(h_t), \pset(h_t)\rangle$. Here, $\nset(h_t)$ counts how often the history $h_t$ has been visited, $\vset(h_t)$ calculates the expected return from all simulations starting at $h_t$, and $\pset(h_t)$ contains particles representing POMDP states that estimate the belief state $b_t$. The algorithm iterates through four main steps:

\revise{
(1) \textbf{Selection:} A state $s$ is selected at random from the particle set $\pset(h_t)$.
(2) \textbf{Simulation:} If $\cT(h)$ is a non-leaf node, an action $a$ is selected to maximize $\vset(ha) + c \sqrt{\frac{\log \nset(h)}{\nset(ha)}}$ using the \emph{upper confidence bound} (UCB)~\cite{auer2002finite} to balance exploration and exploitation. 
(3) \textbf{Expansion:} Upon reaching a leaf node ${\cT(h)}$, new child nodes are introduced for every action $a \in A$, expressed as $\cT(ha)= \langle \nset_{\mathit{init}}(ha), \vset_{\mathit{init}}(ha), \emptyset \rangle$.
Then, an action $a$ is selected based on a predefined rollout policy such as uniform random selection. A state $s'$ is simulated using a black-box simulator $(s',o,r) \sim \cG(s,a)$ and added to $\pset(hao)$. This process runs until reaching a predetermined depth.
(4) \textbf{Backpropagation:} After simulation, the search tree nodes are updated with new data from the path.
}

The planning phase at timestep $t$ concludes once a target number of these iterations has occurred or a time limit is reached. The agent then takes the best action $a_t = \argmax_a\vset(h_ta)$, receives a new observation $o_{t+1}$, and proceeds to the next step, constructing a search tree from the new root node $\cT(h_ta_to_{t+1})$.
POMCP is effective because it mitigates the \emph{curse of dimensionality} through state sampling and the \emph{curse of history} via history sampling with a black-box simulator.

\section{Approach} \label{sec:approach} 

We develop a novel approach that constructs safety shields on-the-fly using ACP-based trajectory predictions for dynamic agents, and shields unsafe actions in POMDP online planning. We first define safety shields in \sectref{sec:shield}, present an algorithm for on-the-fly shield construction in \sectref{sec:construct}, describe the shielding method for safe online planning in \sectref{sec:safe-pomcp}, and analyze the correctness and complexity of the proposed approach in \sectref{sec:correct}.

\subsection{ACP-induced Safety Shield} \label{sec:shield}

Given a $\tau$-step ahead prediction $\hat{X}_t^\tau$ made at timestep $t$ for the dynamic agents' future state, and its corresponding ACP region $C_{t+\tau}^{\tau}$ computed as per \sectref{sec:acp}, we define the safety constraint using the Lipschitz continuous function described in \eqnref{eqn:safety} as follows:
\begin{equation} \label{eqn:safety-acp}
    c(s_{t+\tau}, \hat{X}_t^\tau) \ge L \cdot C_{t+\tau}^{\tau}
\end{equation}
where $L>0$ is the Lipschitz constant and $\tau \in \{1, \dots, H\}$ assuming prediction horizon $H$. 

Denote by $F_t^\tau := \{s \in S \mid c(s, \hat{X}_t^\tau) < L \cdot C_{t+\tau}^{\tau}\}$ the $\tau$-step ahead prediction region of $\hat{X}_t^\tau$ made at timestep $t$. It follows from \cite{dixit2023adaptive} that $Pr({X}_{t+\tau}\in F_t^\tau ) \ge 1-\delta, \forall \tau \in \{1, \dots, H\}$, where ${X}_{t+\tau}$ is the true state of the dynamic agents at timestep $t+\tau$. By ensuring that the prediction regions $\{F_t^\tau\}_{\tau\in \{1, \dots, H\}}$ are avoided almost-surely (i.e., with probability 1) at every timestep $t$, it becomes possible to achieve the desired probabilistic safety guarantee, i.e.,  $\phi^{\pi}(b_0) \ge 1 - \delta$, with regards to the dynamic agents.

\begin{examp}
We have $\hat{X}_t^1=(\langle 17.334, 9.711 \rangle)$ and $C_t^{1} = 0.736$ from \egref{eg:acp}.
Let the Lipschitz constant be $L=1$ and the safety buffer $\epsilon=2$.
Suppose the robot's actual state at timestep $t+1$ is $s_{t+1}=(18,4)$. 
We have $c(s_{t+1}, \hat{X}_t^1) = \|s_{t+1} - \hat{X}_t^1\| - \epsilon = 3.7497 \ge C_t^{1}$.
Thus, the safety constraint is satisfied. 
\end{examp}

In \cite{junges2021enforcing}, it was demonstrated that, for enforcing almost-sure safety specifications, belief probabilities are irrelevant and only the belief support is important. Inspired by this observation, we define a winning belief support and winning regions for a given horizon $h$, which are used for shielding unsafe actions during online planning.

We say that a POMDP policy $\pi$ is \emph{winning} from belief state $b_t$ for a finite horizon $h \le  H$ iff every state in the belief supports $\supp(b_{t+j})$ for $j \in \{0, \dots, h\}$ of all possible executions under the policy $\pi$ satisfies the safety constraint of \eqnref{eqn:safety-acp}.

A belief state $b \in B$ is considered winning for a horizon $h$ if there exists an $h$-step horizon winning policy $\pi$ originating from $b$, and the belief support, denoted as $\supp(b)$, is termed a \emph{winning belief support} for the horizon $h$. 
A set of belief supports, denoted by $W \subseteq S_B$, is termed a \emph{winning region} for an $h$-step horizon if every belief support $\Theta \in W$ is winning for the horizon $h$. 
In the special case when $h=0$, we say that $W$ is a winning region with a zero-step horizon iff all states belonging to each belief support $\supp(b) \in W$ satisfy the safety constraint of \eqnref{eqn:safety-acp}.

To enforce safety, we can define a safety shield, denoted by $\xi: \Theta \to 2^A$, which restricts actions to those leading solely to successor belief supports within the winning region $W$.

\subsection{Computing Winning Regions for Shields} \label{sec:construct}

\begin{algorithm}[t]
\small
\caption{Computing winning regions}\label{alg:wr}
\DontPrintSemicolon
\KwIn{POMDP model $\M$, a winning belief state $b_t$, a set of dynamic agents' predicted states $\{\hat{X}_t^\tau\}_{\tau=1}^H$ and ACP prediction regions $\{C_{t+\tau}^{\tau}\}_{\tau=1}^H$.}
\KwOut{A set of winning regions $\{W_t^\tau\}_{\tau=1}^H$.}
Compute the set of reachable belief supports $S_B^{b_t,H}$ \;
Construct a belief-support transition system $\M_B^{b_t,H}$ \;
Compute the set of unsafe states $\Psi_t$ in $\M_B^{b_t,H}$ \;
$W_t^{H} \gets S_{B}^{b_t, H}\setminus \{\Theta\in S_{B}^{b_t, H} \mid \langle \Theta, H \rangle \in \Psi_t\}$ \;
\For{$\tau=H-1$ \KwTo $1$}{
    \ForEach {$\Theta \in S_B^{b_t,H}$}{
        \If {$\langle \Theta, \tau \rangle \notin \Psi_t$}{
            \ForEach {$a\in A$}{
                \If {$\post(\langle \Theta, \tau \rangle, a) \subseteq W_t^{\tau+1}$}{
                    insert $\Theta$ to $W_t^\tau$  \;
                }
            }
        }
    }
}
\Return{$\{W_t^\tau\}_{\tau=1}^H$}
\end{algorithm}

We present \agref{alg:wr} for computing a set of winning regions, denoted by $\{W_t^\tau\}_{\tau=1}^H$, to construct safety shields on-the-fly at timestep $t$, where each winning region $W_t^\tau$ has a winning horizon of $H-\tau$.

First, given a winning belief state $b_t$, we compute the set of belief supports of the POMDP $\M$ that are reachable within $H$ steps, denoted by $S_B^{b_t,H} \subseteq S_B$. 

Next, we construct the reachable fragment of a belief-support transition system (BSTS) for the POMDP $\M$, 
denoted by a tuple $\M_B^{b_t,H} = \{S_B^{b_t,H} \times Q, \bar{s}_B^{b_t, H}, A, T_B^{b_t, H}\}$, 
where the state space is the product of $S_B^{b_t,H}$ and a time counter $Q = \{0,1,\dots,H\}$,
the initial state is $\bar{s}_B^{b_t, H} = \langle \supp(b_t),0 \rangle$,
the set of actions $A$ are the same as in the POMDP $\M$,
and the transition function is given by 
$T_B^{b_t, H}(\langle \Theta,q \rangle, a) = \langle \Theta',q+1 \rangle$ if  
\(
\Theta' \in \left\{ \bigcup_{s\in \Theta} \{s' \in S \mid T(s,a,s')>0 \text{ and } o \in \obs(s')\} \mid o \in O \right\}.
\) The transition function $T_B^{b_t, H}$ is constructed as follows. Given a belief support $\Theta\in S_B^{b_t,H}$, a timestep $q\in Q$, and an action $a$, our initial step involves calculating the set of all potential successor states from $\Theta$ and $a$. This is accomplished by determining the set of all possible successor states from each state $s\in \Theta$ and then taking the union of these sets. 
Following that, we reorganize these states into a set of belief supports based on the principle that   states within a belief support share the same observation. Finally, the timestep is increased by 1.
Let $\post(\langle \Theta, q \rangle, a) := \{\Theta' \mid T_B^{b_t, H}(\langle \Theta,q \rangle, a) = \langle \Theta',q+1 \rangle\}$ denote the set of all possible successor belief supports.

Given a set of dynamic agents' predicted states $\{\hat{X}_t^\tau\}_{\tau=1}^H$ and ACP regions $\{C_{t+\tau}^{\tau}\}_{\tau=1}^H$, we compute all possible POMDP states where the safety constraint could be violated, denoted as 
$F_t^\tau := \{s \in S \mid c(s, \hat{X}_t^\tau) < L \cdot C_{t+\tau}^{\tau}\}$.
The set of unsafe states in the BSTS $\M_B^{b_t,H}$ is defined as 
$\Psi_t := \{\langle \Theta, q \rangle \in S_B^{b_t,H} \times Q \mid \exists s \in \Theta \text{ such that } s \in F_t^q\}$.

We compute the set of winning regions $\{W_t^\tau\}_{\tau=1}^H$ recursively in a backward manner.
Let $W_t^{H}$ be the set of reachable belief supports $\Theta \in S_B^{b_t, H}$, excluding those that result in unsafe states where $\langle \Theta, H \rangle \in \Psi_t$.
Starting from $\tau=H-1$, we add a reachable belief support $\Theta \in S_B^{b_t, H}$ to the winning region $W_t^\tau$ only if both of the following two conditions hold:
(C1) $\langle \Theta, \tau \rangle$ does not belong to the unsafe set $\Psi_t$;
and (C2) there exists an action $a \in A$ that leads solely to winning successor belief supports $\post(\langle \Theta, \tau \rangle, a) \subseteq W_t^{\tau+1}$ from $\Theta$ in the BSTS $\M_B^{b_t,H}$.

\begin{examp}
Following previous examples, we compute the set of unsafe POMDP states for a one-step prediction at timestep $t$ as 
$F_t^1 = \{s \in S \mid c(s, \hat{X}_t^1) < C_{t+1}^{1}\}$. 
These states are represented by the yellow shadow in \figref{fig:example}. 
Similarly, for a two-step prediction at timestep $t$, we compute the set of unsafe states as 
$F_t^2 = \{s \in S \mid c(s, \hat{X}_t^2) < C_{t+2}^{2}\}$. 
The winning region $W_t^2$ is identified as the set of belief supports $S_{B}^{b_t, 2}$ that can be reached from $b_t$ within two steps, while excluding those that contain any unsafe states from $F_t^2$.
We say that $W_t^2$ has a zero-step winning horizon, because we do not evaluate the safety of actions leading to states beyond the prediction horizon $H=2$. 
The one-step horizon winning region $W_t^1$ is identified as the set of reachable belief supports that do not contain unsafe states from $F_t^1$ and can lead solely to successor belief supports in $W_t^2$.
\end{examp}

\subsection{Safe Online Planning via Shielding} \label{sec:safe-pomcp} 

\agref{alg:shield} illustrates the proposed safe POMDP online planning approach. 
At each planning step $t$, it first predicts dynamic agents' trajectories $\{\hat{X}_t^\tau\}_{\tau=1}^H$ and computes ACP regions $\{C_{t+\tau}^{\tau}\}_{\tau=1}^H$ as described in \sectref{sec:acp}.
Then, \agref{alg:wr} is applied to compute a set of winning regions $\{W_t^\tau\}_{\tau=1}^H$ for the safety shield. 

In Line 5, the procedure $\mathsf{shieldPOMCP}$ is called, integrating shields with the POMCP algorithm (see \sectref{sec:pomcp}), which navigates a search tree whose root node is $\cT(h_t)$. 
During the simulation phase of the POMCP algorithm, when an action $a \in A$ is selected (either by the UCB rule or during rollout) for the history $h_{t+\tau-1}$, with $\tau \in \{1, \dots, H\}$, and a black-box simulator generates $(s',o,r) \sim \cG(s,a)$, the procedure checks if the particle set $\pset(h_{t+\tau-1}ao) \cup \{s'\}$ belongs to the winning region $W_t^\tau$. 
If the resulting particle set is not a winning belief support, the branch of the tree starting from node $\cT(h_{t+\tau-1}a)$ is pruned, effectively shielding action $a$ at node $\cT(h_{t+\tau-1})$.
If the simulation depth exceeds $H$, no shield is applied to actions selected for any history beyond $h_{t+H}$. 

When the POMCP planning concludes at timestep $t$, the best action $a_t$ is selected from the set of allowed actions at node $\cT(h_t)$ as the one that achieves the maximum value of $\vset(h_ta)$.
We set the policy $\pi^*$ with $\pi^*(b_t)=a_t$. 
The agent executes $a_t$, receives an observation $o_{t+1}$ and updates the belief state $b_{t+1}$ for the next step. 

\begin{algorithm}[t]
\small
\caption{Safe online planning via shielding}\label{alg:shield}
\DontPrintSemicolon
\KwIn{POMDP model $\M$, an initial belief state $b_0$, a distribution $\D$ of dynamic agents' trajectories, a prediction horizon $H$, and a failure probability $\delta$.}
\KwOut{A safe POMDP policy $\pi^*$.}
\For{$t=0$ \KwTo $\infty$}{
    predict dynamic agents' trajectories $\{\hat{X}_t^\tau\}_{\tau=1}^H \sim \D$ \;
    compute ACP prediction regions $\{C_{t+\tau}^{\tau}\}_{\tau=1}^H$ w.r.t failure rate $\delta$  \;
    compute winning regions $\{W_t^\tau\}_{\tau=1}^H$ \tcp*[r]{\agref{alg:wr}}
    $\pi^*(b_t) \gets \mathsf{shieldPOMCP}(\cT(h_t),\{W_t^\tau\}_{\tau=1}^H)$ \;
    $(o_{t+1}, b_{t+1}) \gets$ execute action $\pi^*(b_t)$
}
\Return{$\pi^*$}
\end{algorithm}

\begin{figure*}[t]
    \centering
    \includegraphics[width=\linewidth]{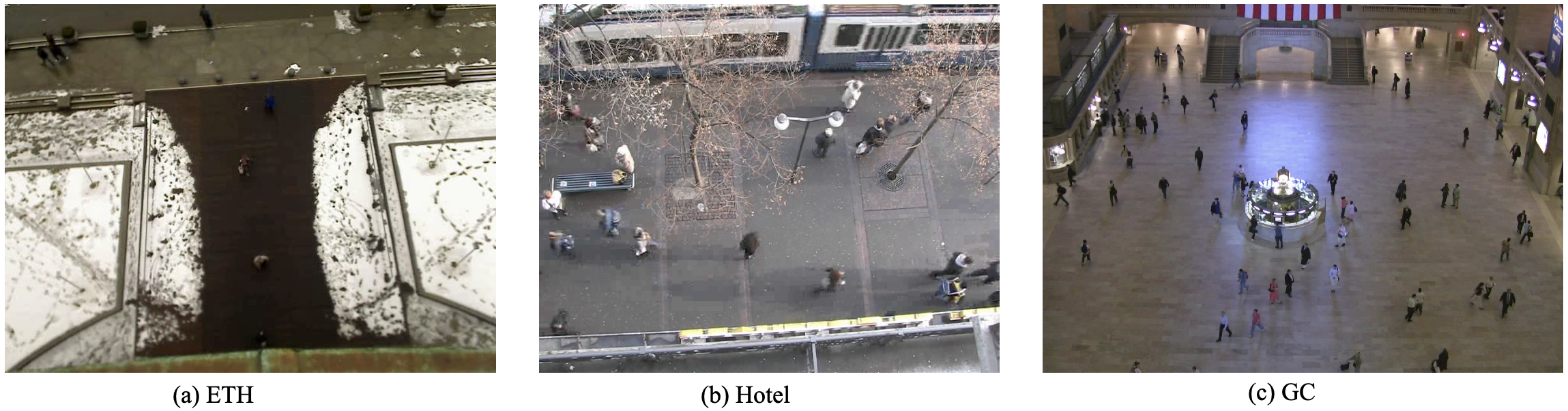}
    \caption{Example scenes of real-world pedestrians trajectories from benchmark datasets~\cite{amirian2020opentraj}.}
    \label{fig:datasets}
\end{figure*}

\begin{examp}
At timestep $t$, consider a state $s= \langle 17,5 \rangle$ sampled from the particle set $\pset(h_t)$ at the root node $\cT(h_t)$. During the POMCP simulation, suppose action $a=\emph{east}$ is chosen for history $h_t$, and a black-box simulator yields $(s',o,r) \sim \cG(s,a)$ with $s' = \langle 18,5 \rangle$. We then check if $\pset(h_tao) \cup \{s'\}$ falls within the winning region $W_t^1$; it does, so the simulation process advances with action $a'$ for history $h_tao$. Assuming $a'=\emph{north}$ leads to $(s'',o',r') \sim \cG(s',a')$ with $s'' = \langle 18,7 \rangle$, an unsafe state in $F_t^2$, the updated particle set $\pset(h_taoa'o') \cup \{s''\}$ falls outside of $W_t^2$. Therefore, action \emph{north} is shielded at node $\cT(h_tao)$.
\end{examp}

\subsection{Correctness and Complexity} \label{sec:correct}

\startpara{Correctness}
The correctness of \agref{alg:wr} is stated in \lemref{lem:wr} and the correctness of \agref{alg:shield}, with respect to the problem statement in \sectref{sec:problem}, is stated in \thmref{thm:main}. The proofs are given in the Appendix. 

\begin{lemma}\label{lem:wr}
The output of \agref{alg:wr}, denoted by $\{W_t^\tau\}_{\tau=1}^H$, comprises a set of winning regions, with each $W_t^\tau$ representing a winning region for an $(H-\tau)$-step horizon.
\end{lemma}

\begin{theorem}\label{thm:main}
Given a POMDP model $\M$ for a robotic agent with initial belief state $b_0$, the unknown random trajectories $X\sim \D$ of $N$ dynamic agents with a prediction horizon $H$, and a failure probability $\delta \in (0,1)$, the policy $\pi^*$ computed by \agref{alg:shield} achieves the maximal expected return $V^{\pi^*}(b_0)$ while ensuring safety, i.e., $\phi^{\pi^*}(b_0) \ge 1 - \delta$. 
\end{theorem}

\revise{It is important to note that, since POMCP is a sampling-based algorithm, there is a risk that approximate belief states might compromise safety guarantees. However, given a sufficiently large particle set and a substantial number of simulations, POMCP can yield near-perfect belief estimates (at the cost of increased computational burden, as detailed in the complexity analysis below). The above theorem operates under this assumption.}

\startpara{Complexity}
There are several components in the complexity analysis of the proposed approach in \agref{alg:shield}. 
The complexity of predicting dynamic agents' trajectories depends on the underlying prediction model. 
The complexity of computing ACP regions at timestep $t$ is $\O(H\cdot N\cdot K \cdot \log(K))$, depending on the prediction horizon $H$, the number of agents $N$, and the sliding window size $K$.
The complexity of \agref{alg:wr} for computing winning regions to construct shields is bounded by $\O(H \cdot |S_B^{b_t,H}| \cdot |A|)$, which depends on the prediction horizon $H$, the number of $H$-step reachable belief supports $S_B^{b_t,H}$ from belief state $b_t$, and the size of the action set $A$. \revise{The complexity of POMCP involves both time and space dimensions and is scenario-dependent. It is influenced by parameters such as the size of the particle set (i.e., the number of particles used for approximating a belief state), the number of simulations, the simulation depth, and the size of the state and observation spaces. } The overhead of adding shielding to POMCP, specifically checking against safety shields, is bounded by $\O\left((|A|\cdot|O|)^H\right)$ per simulation. 
In practice, the overhead is negligible, as demonstrated by the experimental results in the next section.

\section{Experiments} \label{sec:exp} 
\begin{table*}[t]
\caption{Experiment Results}
\resizebox{1\textwidth }{!}{
\begin{tabular}{c|cccc|cccc|cccc}
\toprule
 &
  \multicolumn{4}{c|}{ETH} &
  \multicolumn{4}{c|}{Hotel} &
  \multicolumn{4}{c}{GC} \\ \midrule
Method &
  $N$ &
  Safety Rate &
  Time (s) &
  Min Distance  &
  $N$ &
  Safety Rate &
  Time (s) &
  Min Distance  &
  $N$ &
  Safety Rate &
  Time (s) &
  Min Distance  \\ \midrule
No Shield &
  \multirow{3}{*}{45} &
  0.893 &
  21.1 &
  0.28$\pm$0.19 &
  \multirow{3}{*}{35} &
  0.944 &
  20.1 &
  0.42$\pm$0.21 &
  \multirow{3}{*}{160} &
  0.91 &
  39.3 &
  0.22$\pm$0.13 \\
Shielding without ACP &
   &
  0.943 &
  21.5 &
  0.39$\pm$0.29 &
   &
  0.969 &
  20.3 &
  0.54$\pm$0.3 &
   &
  0.943 &
  67.2 &
  0.23$\pm$0.13 \\
Shielding with ACP &
   &
  \textbf{0.974} &
  22.1 &
  0.51$\pm$0.29 &
   &
  \textbf{0.988} &
  20.6 &
  0.8$\pm$0.49 &
   &
  \textbf{0.963} &
  71.4 &
  0.28$\pm$0.16 \\ \midrule
No Shield &
  \multirow{3}{*}{55} &
  0.891 &
  21.0 &
  0.26$\pm$0.17 &
  \multirow{3}{*}{45} &
  0.931 &
  20.1 &
  0.38$\pm$0.24 &
  \multirow{3}{*}{180} &
  0.904 &
  39.9 &
  0.2$\pm$0.1 \\
Shielding without ACP &
   &
  0.951 &
  21.8 &
  0.41$\pm$0.25 &
   &
  0.959 &
  20.1 &
  0.48$\pm$0.24 &
   &
  0.938 &
  66.8 &
  0.23$\pm$0.12 \\
Shielding with ACP &
   &
  \textbf{0.975} &
  22.4 &
  0.53$\pm$0.37 &
   &
  \textbf{0.982} &
  20.6 &
  0.62$\pm$0.27 &
   &
  \textbf{0.953} &
  71.1 &
  0.24$\pm$0.16 \\ \midrule
No Shield &
  \multirow{3}{*}{65} &
  0.872 &
  21.2 &
  0.24$\pm$0.13 &
  \multirow{3}{*}{55} &
  0.921 &
  20.2 &
  0.36$\pm$0.18 &
  \multirow{3}{*}{200} &
  0.895 &
  39.6 &
  0.22$\pm$0.11 \\
Shielding without ACP &
   &
  0.943 &
  21.9 &
  0.36$\pm$0.2 &
   &
  0.957 &
  20.3 &
  0.48$\pm$0.29 &
   &
  0.931 &
  65.7 &
  0.2$\pm$0.13 \\
Shielding with ACP &
   &
  \textbf{0.967} &
  22.6 &
  0.42$\pm$0.26 &
   &
  \textbf{0.982} &
  20.3 &
  0.6$\pm$0.24 &
   &
  \textbf{0.951} &
  74.3 &
  0.25$\pm$0.15 \\ \bottomrule
\end{tabular}
\label{tab:exp}
}
\end{table*}

We implemented the proposed approach and evaluated it through various computational experiments.
All experiments were run on a MacBook Pro machine with 10-core 3.2 GHz Apple M1 processor and 16 GB of memory. 

\startpara{Environments}
We consider three gridworld environments, where the movements of $N$ dynamic agents follow real-world pedestrians trajectories derived from benchmark datasets~\cite{amirian2020opentraj}. Specifically, we use three datasets: ETH, Hotel, and GC, for which example scenes are illustrated in \figref{fig:datasets}.
In each gridworld environment, the robot aims to reach a target destination while avoiding pedestrians. The robot can move \emph{east}, \emph{south}, \emph{west}, or \emph{north}, reaching the adjacent cell with probability 0.1 or one cell further with probability 0.9. \revise{The robot has partial observability of its position within a $2\times2$ block but is unsure of the exact grid cell due to noisy sensors.} The reward function is defined as: $1,000$ for reaching the destination, $-1$ per step, and $-10$ per collision. 

\startpara{Hyperparameters}
Each pedestrian trajectory dataset is split into training, validation, and test sets in a ratio of 16:4:5. 
LSTM models with 64 hidden units are trained for the trajectory prediction with a prediction horizon $H=3$. 
We set the failure probability $\delta=0.05$ and the safety buffer to be $\epsilon=0.5$ grid. 
The ACP prediction regions are computed with a learning rate $\alpha=0.0008$.
The POMCP algorithm is set with the following hyperparameters: the number of simulations is 4,096; the simulation depth is 200; and the number of particles sampled from the initial state distribution is 10,000. 

\startpara{Results}
We compare the performance of the proposed approach for safe online planning via ACP-induced shields with two baselines: (i) \emph{No Shield}, i.e., POMCP without shielding; and (ii) \emph{Shielding without ACP}, i.e., POMCP with shields that do not account for ACP prediction regions, by modifying the safety constraint in \eqnref{eqn:safety-acp} to $c(s_{t+\tau}, \hat{X}_t^\tau) \ge 0$.

\tabref{tab:exp} shows the results averaged over 100 runs of each method, for different environments and varying numbers $N$ of dynamic agents. 
Across all cases, the proposed approach achieves a better safety rate (measured by the percentage of time satisfying the safety constraint during a run) than the baselines.
All three methods result in comparable travel times for the robot to reach the destination in the ETH and Hotel environments, while shielding approaches lead to longer travel times when the robot needs to avoid a significantly higher number of pedestrians in the GC environment.
\tabref{tab:exp} also reports the mean and standard deviation of the minimum distance between the robot and pedestrians; the proposed approach is more conservative than the baselines, maintaining a larger minimum distance for safety.

Finally, we observed that shielding does not significantly add to the runtime of online planning. The average computation time for each planning step of POMCP without shields is 0.28 seconds, compared to 0.35 seconds for the proposed approach, incurring an additional 0.07 seconds for constructing winning regions and shielding actions per step.

\section{Conclusion} \label{sec:conclu} 
This work developed a novel shielding approach aimed at ensuring safe POMDP online planning in dynamic environments with multiple unknown dynamic agents. We proposed to leverage ACP for predicting the future trajectories of these dynamic agents. Subsequently, safety shields are computed based on the prediction regions. Finally, we integrated these safety shields into the POMCP algorithm to enable safe POMDP online planning. Experimental results conducted on three benchmark domains, each with varying numbers of dynamic agents, demonstrated that the proposed approach successfully met the safety requirements while adhering to a predefined failure rate.

For future work, we will assess the effectiveness of the proposed approach across diverse POMDP domains and deploy it in real-world robotic tasks. Another important direction is to explore methods for handling continuous POMDPs to enhance scalability and generalization.


\bibliographystyle{IEEEtran}
\bibliography{references}

\begin{thebibliography}{10}
\providecommand{\url}[1]{#1}
\csname url@samestyle\endcsname
\providecommand{\newblock}{\relax}
\providecommand{\bibinfo}[2]{#2}
\providecommand{\BIBentrySTDinterwordspacing}{\spaceskip=0pt\relax}
\providecommand{\BIBentryALTinterwordstretchfactor}{4}
\providecommand{\BIBentryALTinterwordspacing}{\spaceskip=\fontdimen2\font plus
\BIBentryALTinterwordstretchfactor\fontdimen3\font minus \fontdimen4\font\relax}
\providecommand{\BIBforeignlanguage}[2]{{%
\expandafter\ifx\csname l@#1\endcsname\relax
\typeout{** WARNING: IEEEtran.bst: No hyphenation pattern has been}%
\typeout{** loaded for the language `#1'. Using the pattern for}%
\typeout{** the default language instead.}%
\else
\language=\csname l@#1\endcsname
\fi
#2}}
\providecommand{\BIBdecl}{\relax}
\BIBdecl

\bibitem{lauri2022partially}
M.~Lauri, D.~Hsu, and J.~Pajarinen, ``Partially observable {Markov} decision processes in robotics: A survey,'' \emph{IEEE Transactions on Robotics}, vol.~39, no.~1, pp. 21--40, 2022.

\bibitem{sheng2022planning}
S.~Sheng, E.~Pakdamanian, K.~Han, Z.~Wang, J.~Lenneman, D.~Parker, and L.~Feng, ``Planning for automated vehicles with human trust,'' \emph{ACM Transactions on Cyber-Physical Systems}, vol.~6, no.~4, pp. 1--21, 2022.

\bibitem{yu2024trust}
P.~Yu, S.~Dong, S.~Sheng, L.~Feng, and M.~Kwiatkowska, ``Trust-aware motion planning for human-robot collaboration under distribution temporal logic specifications,'' in \emph{International Conference on Robotics and Automation}, 2024.

\bibitem{silver2010monte}
D.~Silver and J.~Veness, ``{Monte-Carlo} planning in large {POMDP}s,'' in \emph{Advances in Neural Information Processing Systems}, 2010.

\bibitem{lee2018monte}
J.~Lee, G.-H. Kim, P.~Poupart, and K.-E. Kim, ``{Monte-Carlo} tree search for constrained {POMDP}s,'' in \emph{Advances in Neural Information Processing Systems}, 2018.

\bibitem{khonji2019approximability}
M.~Khonji, A.~Jasour, and B.~C. Williams, ``Approximability of constant-horizon constrained {POMDP},'' in \emph{International Joint Conference on Artificial Intelligence}, 2019, pp. 5583--5590.

\bibitem{sheng2024safe}
S.~Sheng, D.~Parker, and L.~Feng, ``Safe {POMDP} online planning via shielding,'' in \emph{International Conference on Robotics and Automation}, 2024.

\bibitem{dixit2023adaptive}
A.~Dixit, L.~Lindemann, S.~X. Wei, M.~Cleaveland, G.~J. Pappas, and J.~W. Burdick, ``Adaptive conformal prediction for motion planning among dynamic agents,'' in \emph{Learning for Dynamics and Control Conference}, 2023, pp. 300--314.

\bibitem{wang2021online}
Y.~Wang, A.~A.~R. Newaz, J.~D. Hern{\'a}ndez, S.~Chaudhuri, and L.~E. Kavraki, ``Online partial conditional plan synthesis for {POMDP}s with safe-reachability objectives: Methods and experiments,'' \emph{IEEE Transactions on Automation Science and Engineering}, vol.~18, no.~3, pp. 932--945, 2021.

\bibitem{mazzi2023risk}
G.~Mazzi, A.~Castellini, and A.~Farinelli, ``Risk-aware shielding of partially observable {Monte-Carlo} planning policies,'' \emph{Artificial Intelligence}, vol. 324, p. 103987, 2023.

\bibitem{kurniawati2016online}
H.~Kurniawati and V.~Yadav, ``An online {POMDP} solver for uncertainty planning in dynamic environment,'' in \emph{The 16th International Symposium on Robotics Research}, 2016, pp. 611--629.

\bibitem{bauer2009autonomous}
A.~Bauer, K.~Klasing, G.~Lidoris, Q.~M{\"u}hlbauer, F.~Rohrm{\"u}ller, S.~Sosnowski, T.~Xu, K.~K{\"u}hnlenz, D.~Wollherr, and M.~Buss, ``The autonomous city explorer: Towards natural human-robot interaction in urban environments,'' \emph{International Journal of Social Robotics}, vol.~1, pp. 127--140, 2009.

\bibitem{kummerle2015autonomous}
R.~K{\"u}mmerle, M.~Ruhnke, B.~Steder, C.~Stachniss, and W.~Burgard, ``Autonomous robot navigation in highly populated pedestrian zones,'' \emph{Journal of Field Robotics}, vol.~32, no.~4, pp. 565--589, 2015.

\bibitem{luo2018porca}
Y.~Luo, P.~Cai, A.~Bera, D.~Hsu, W.~S. Lee, and D.~Manocha, ``Porca: Modeling and planning for autonomous driving among many pedestrians,'' \emph{IEEE Robotics and Automation Letters}, vol.~3, no.~4, pp. 3418--3425, 2018.

\bibitem{zhu2023gaussian}
E.~L. Zhu, F.~L. Busch, J.~Johnson, and F.~Borrelli, ``A gaussian process model for opponent prediction in autonomous racing,'' in \emph{2023 IEEE/RSJ International Conference on Intelligent Robots and Systems (IROS)}, 2023, pp. 8186--8191.

\bibitem{farid2023task}
A.~Farid, S.~Veer, B.~Ivanovic, K.~Leung, and M.~Pavone, ``Task-relevant failure detection for trajectory predictors in autonomous vehicles,'' in \emph{Conference on Robot Learning}, 2023, pp. 1959--1969.

\bibitem{angelopoulos2023conformal}
A.~N. Angelopoulos, S.~Bates \emph{et~al.}, ``Conformal prediction: A gentle introduction,'' \emph{Foundations and Trends{\textregistered} in Machine Learning}, vol.~16, no.~4, pp. 494--591, 2023.

\bibitem{zaffran2022adaptive}
M.~Zaffran, O.~F{\'e}ron, Y.~Goude, J.~Josse, and A.~Dieuleveut, ``Adaptive conformal predictions for time series,'' in \emph{International Conference on Machine Learning}, 2022, pp. 25\,834--25\,866.

\bibitem{lindemann2023safe}
L.~Lindemann, M.~Cleaveland, G.~Shim, and G.~J. Pappas, ``Safe planning in dynamic environments using conformal prediction,'' \emph{IEEE Robotics and Automation Letters}, 2023.

\bibitem{sun2024conformal}
J.~Sun, Y.~Jiang, J.~Qiu, P.~Nobel, M.~J. Kochenderfer, and M.~Schwager, ``Conformal prediction for uncertainty-aware planning with diffusion dynamics model,'' \emph{Advances in Neural Information Processing Systems}, 2024.

\bibitem{moss2024constrainedzero}
R.~J. Moss, A.~Jamgochian, J.~Fischer, A.~Corso, and M.~J. Kochenderfer, ``{ConstrainedZero: Chance-constrained {POMDP} planning using learned probabilistic failure surrogates and adaptive safety constraints},'' in \emph{International Joint Conference on Artificial Intelligence}, 2024.

\bibitem{ren2023robots}
A.~Z. Ren, A.~Dixit, A.~Bodrova, S.~Singh, S.~Tu, N.~Brown, P.~Xu, L.~Takayama, F.~Xia, J.~Varley \emph{et~al.}, ``Robots that ask for help: Uncertainty alignment for large language model planners,'' in \emph{Conference on Robot Learning}, 2023, pp. 661--682.

\bibitem{coulom2006efficient}
R.~Coulom, ``Efficient selectivity and backup operators in {Monte-Carlo} tree search,'' in \emph{International Conference on Computers and Games}, 2006, pp. 72--83.

\bibitem{auer2002finite}
P.~Auer, N.~Cesa-Bianchi, and P.~Fischer, ``Finite-time analysis of the multiarmed bandit problem,'' \emph{Machine learning}, vol.~47, pp. 235--256, 2002.

\bibitem{junges2021enforcing}
S.~Junges, N.~Jansen, and S.~A. Seshia, ``Enforcing almost-sure reachability in {POMDP}s,'' in \emph{International Conference on Computer Aided Verification}, 2021, pp. 602--625.

\bibitem{amirian2020opentraj}
J.~Amirian, B.~Zhang, F.~V. Castro, J.~J. Baldelomar, J.-B. Hayet, and J.~Pettre, ``Opentraj: Assessing prediction complexity in human trajectories datasets,'' in \emph{Asian Conference on Computer Vision}, 2020.

\end{thebibliography}

\appendix \label{sec:app} 

\setcounter{lemma}{0}
\begin{lemma}
The output of \agref{alg:wr}, denoted by $\{W_t^\tau\}_{\tau=1}^H$, comprises a set of winning regions, with each $W_t^\tau$ representing a winning region for an $(H-\tau)$-step horizon.
\end{lemma}


\begin{proof}
We prove the lemma by induction on the increasing horizon of winning regions.

\uitstart{Base case}:
By the definition of $W_t^{H}$,
for any belief state $b_{t+H}$ whose belief support $\supp(b_{t+H}) \in W_t^{H}$, every state $s_{t+H} \in \supp(b_{t+H})$ satisfies the safety constraint, i.e.,
$c(s_{t+H}, \hat{X}_t^H) \ge L \cdot C_{t+H}^{H}$.
Thus, $W_t^{H}$ is a winning region for zero-step horizon.

\uitstart{Inductive step}:
Assume that $W_t^{\tau}$ is the winning region for an $(H-\tau)$-step horizon. By definition, for every belief state $b_{t+\tau}$ whose belief support $\supp(b_{t+\tau}) \in W_t^{\tau}$, there exists an $(H-\tau)$-step horizon winning policy $\pi_{t+\tau}$ originating from $b_{t+\tau}$. For every belief state $b_{t+\tau-1}$ with belief support $\supp(b_{t+\tau-1}) \in W_t^{\tau-1}$, there exists at least one action $a_{t+\tau-1} \in A$ that leads to a successor belief support state 
\[\supp(b_{t+\tau}) \in \post(\langle \supp(b_{t+\tau-1}), \tau-1 \rangle, a_{t+\tau-1}) \subseteq W_t^\tau.\]
We can augment policy $\pi_{t+\tau}$ into an $(H-\tau+1)$-step horizon winning policy $\pi_{t+\tau-1}$ with $\pi_{t+\tau-1}(b_{t+\tau-1})=a_{t+\tau-1}$. Thus, $W_t^{\tau-1}$ is the winning region for an $(H-\tau+1)$-step horizon.

\uitstart{Conclusion}: By induction, we have proved that \agref{alg:wr} outputs a set of winning regions $\{W_t^\tau\}_{\tau=1}^H$, with each $W_t^\tau$ having an $(H-\tau)$-step winning horizon.
\end{proof}

\setcounter{theorem}{0}
\begin{theorem}
Given a POMDP model $\M$ for a robotic agent with initial belief state $b_0$, the unknown random trajectories $X\sim \D$ of $N$ dynamic agents with a prediction horizon $H$, and a failure probability $\delta \in (0,1)$, the policy $\pi^*$ computed by \agref{alg:shield} achieves the maximal expected return $V^{\pi^*}(b_0)$ while ensuring safety, i.e., $\phi^{\pi^*}(b_0) \ge 1 - \delta$. 
\end{theorem}


\begin{proof}
Let $a_t = \pi^*(b_t)$ denote the action selected by the policy $\pi^*$ at timestep $t$ computed via \agref{alg:shield}. 
By construction, $a_t$ is a safe action leading solely to successor belief supports within the winning region $W_t^1$, which, according to \lemref{lem:wr}, has an $(H-1)$-step winning horizon.
Thus, for each state $s_{t+1} \in \supp(b_{t+1})$, we have 
\begin{equation} \label{eqn:proof-1}
    c(s_{t+1}, \hat{X}_t^1) \ge L \cdot C_{t+1}^{1}.
\end{equation}
Based on \eqnref{eqn:proof-1} and thanks to Lipschitz continuity of function $c$, we derive that 
\begin{align*}
    0 &\leq  c(s_{t+1}, \hat{X}_t^1) - L \cdot C^1_{t+1} \\
    &\leq c(s_{t+1}, X_{t+1}) + L \cdot \|X_{t+1} - \hat{X}_t^{1}\| - L \cdot C^1_{t+1}.
\end{align*}
Hence, $\|X_{t+1} - \hat{X}_t^{1}\| \leq C^1_{t+1}$ is a sufficient condition for $c(s_{t+1}, X_{t+1}) \geq 0$, that is, $\Pr(c(s_{t+1}, X_{t+1}) \geq 0) \mid \| X_{t+1} - \hat{X}_t^1 \| \leq C_{t+1}^1) = 1.$
Thanks to the law of total probability, we have
\begin{equation} \label{eqn1}
    \Pr(c(s_{t+1}, X_{t+1}) \geq 0) \geq \Pr(\| X_{t+1} - \hat{X}_t^1 \| \leq C_{t+1}^1).
\end{equation} 
With the assistance of Corollary 3 in~\cite{dixit2023adaptive}, it can be shown that the ACP prediction regions $C_{t+1}^{1}$ computed as per \sectref{sec:acp} guarantee that
\begin{equation}\label{eqn2}
     \frac{1}{T}  \sum_{t=0}^{T-1} \Pr(\| X_{t+1} - \hat{X}_t^1 \|\le C_{t+1}^1) \geq 1 - \delta - p_1
\end{equation}
with constant $p_1:=\frac{\lambda^1_0 + \alpha}{T \cdot \alpha}$, where $\lambda^1_0$ is the constant initial value in \eqnref{eq:failureupdate} and $T$ is the number of times \eqnref{eq:failureupdate} is applied. 
Combining \eqneqnref{eqn1}{eqn2}, we have
\begin{equation}\label{eqn3}
     \frac{1}{T}  \sum_{t=0}^{T-1}  \Pr(c(s_{t+1}, X_{t+1}) \geq 0) \geq 1 - \delta - p_1.
\end{equation}
By definition of \eqnref{eqn:prob-safe}, we have $ \rho(b_{t+1}, X_{t+1}) = \sum_{s \in \supp(b_{t+1})} b_{t+1}(s) \cdot \indicator_{\{c(s, X_{t+1}) \ge 0\}}.$
Since \eqnref{eqn3} holds for all $s\in supp(b_{t+1})$, we have
\begin{equation*}
\begin{aligned}
\ \frac{1}{T}  \sum_{t=0}^{T-1} \rho(b_t, X_t)
= &\ \frac{1}{T}  \sum_{t=0}^{T-1} b_{t+1}(s) \cdot \Pr(c(s_{t+1}, X_{t+1}) \geq 0)\\
\ge &\ 1-\delta -p_1.
\end{aligned}
\end{equation*}
Since $\lim_{T \to \infty} p_1=0$, we have
\begin{equation*}
    \phi^{\pi}(b_0) = \mathbb{E}_{\pi} [\lim_{T\to \infty}\frac{1}{T}\sum_{t=0}^{T-1} \rho(b_t, X_t) \ | \ b_0, X_t \sim \D] \ge 1-\delta.
\end{equation*}

Moreover, $a_t = \pi^*(b_t)$ is selected as the best action that achieves the maximum value of $\vset(h_ta)$, which calculates the expected return from all simulations starting at $h_t$, among all safe actions enabled at timestep $t$. 
Thus, given a substantial number of simulations, the expected return $V^{\pi^*}(b_0)$ at the initial belief state $b_0$ is optimal.  

In conclusion, we have proved that the policy $\pi^*$ computed by \agref{alg:shield} achieves the maximal expected return $V^{\pi^*}(b_0)$ while ensuring safety, i.e., $\phi^{\pi^*}(b_0) \ge 1 - \delta$. 

\end{proof}

\end{document}